\documentclass{article}

    \usepackage[preprint]{neurips_2025}

\usepackage[utf8]{inputenc} % allow utf-8 input
\usepackage[T1]{fontenc}    % use 8-bit T1 fonts
\usepackage{hyperref}       % hyperlinks
\usepackage{url}            % simple URL typesetting
\usepackage{booktabs}       % professional-quality tables
\usepackage{amsfonts}       % blackboard math symbols
\usepackage{nicefrac}       % compact symbols for 1/2, etc.
\usepackage{microtype}      % microtypography
\usepackage{xcolor}         % colors

\usepackage{amsmath}
\usepackage{amssymb}
\usepackage{mathtools}
\usepackage{algorithm}
\usepackage{algorithmicx}
\usepackage{algpseudocode}
\usepackage{appendix}
\usepackage{graphicx}
\usepackage{multirow}
\usepackage{diagbox}
\usepackage{arydshln}
\usepackage{caption}

\newtheorem{theorem}{Theorem}
\newtheorem{proof}{Proof}

\newcommand{\down}{\text{down}}

\newcommand{\tabincell}[2]{\begin{tabular}{@{}#1@{}}#2\end{tabular}}

\title{MoLAE: Mixture of Latent Experts for Parameter-Efficient Language Models}
\author{
Zehua Liu, Han Wu, Ruifeng She, Xiaojin Fu, Xiongwei Han, Tao Zhong, Mingxuan Yuan \\Huawei Noah's Ark Lab \\
     \texttt{liuzehua@connect.hku.hk} 
}

\begin{document}

\maketitle

\begin{abstract}
Mixture of Experts (MoE) has become a key architectural paradigm for efficiently scaling Large Language Models (LLMs) by selectively activating a subset of parameters for each input token. However, standard MoE architectures face significant challenges, including high memory consumption and communication overhead during distributed training. In this paper, we introduce Mixture of Latent Experts (MoLAE), a novel parameterization that addresses these limitations by reformulating expert operations through a shared projection into a lower-dimensional latent space, followed by expert-specific transformations. This factorized approach substantially reduces parameter count and computational requirements, particularly in existing LLMs where hidden dimensions significantly exceed MoE intermediate dimensions. We provide a rigorous mathematical framework for transforming pre-trained MoE models into MoLAE architecture, characterizing conditions for optimal factorization, and developing a systematic two-step algorithm for this conversion. Our comprehensive theoretical analysis demonstrates that MoLAE significantly improves efficiency across multiple dimensions while preserving model capabilities. Experimental results confirm that MoLAE achieves comparable performance to standard MoE with substantially reduced resource requirements. 
\end{abstract}

\section{Introduction} \label{sec:intro}

Large Language Models (LLMs) have demonstrated remarkable capabilities across diverse natural language processing tasks \citep{bommasani_2021_opportunities, zhuang_2020_comprehensive}, from text generation \citep{achiam_2023_gpt,dubey_2024_llama} to complex reasoning \citep{guo_2025_deepseek}. As these models scale to increasingly larger parameter spaces, the Mixture of Experts (MoE) architecture \citep{jacobs_1991_adaptive,jordan_1994_hierarchical} has emerged as a promising paradigm for efficiently scaling model capacity without proportionally increasing computational costs. By selectively activating only a subset of parameters for each input token, MoE models achieve parameter efficiency while maintaining manageable inference latency.

Despite their theoretical and empirical advantages, standard MoE architectures \citep{dai_2024_deepseekmoe} face significant practical limitations that inhibit broader deployment. These models require substantial memory resources to store parameters across numerous expert modules in Feed-Forward Network (FFN) layers and create communication bottlenecks during distributed training due to all-to-all data transfers. These challenges become increasingly pronounced as models scale to hundreds of experts, potentially limiting their applicability in resource-constrained environments.
Through systematic investigation of parameter utilization in MoE architectures, we discover substantial redundancy within the FFN layers of current MoE models. By analyzing Qwen1.5-MoE-A2.7B \citep{qwen_2024_qwenmoe}, we reveal that a significant proportion of parameters in its FFN layers can be effectively approximated through lower-dimensional representations without compromising model performance. This empirical observation motivates a fundamental rethinking of expert parameterization in neural architectures.

% \begin{figure}[t!]
% \centering
% \includegraphics[width=0.45\textwidth]{imgs/moe.png}
% \includegraphics[width=0.45\textwidth]{imgs/mole.png} 
% \caption{Comparisons of MoE and MoLAE architectures in the FFN layer. Here, $N$ represents the number of experts in the FFN layer. Compared with MoE, MoLAE introduces the latent mappings $B_{\up}, B_{\gate}$ and $B_{\down}$ to capture the expert shared information. The expert specific information is stored in the expert-specific mappings $A_{\up}^i, A_{\down}^i$ and $A_{\gate}^i$. By introducing the latent space, the MoLAE architecture significantly reduces the memory and computation cost in the FFN layers.}
% \label{fig:1}
% \end{figure}

\begin{figure}[t!]
\centering
\includegraphics[width=1.0\textwidth]{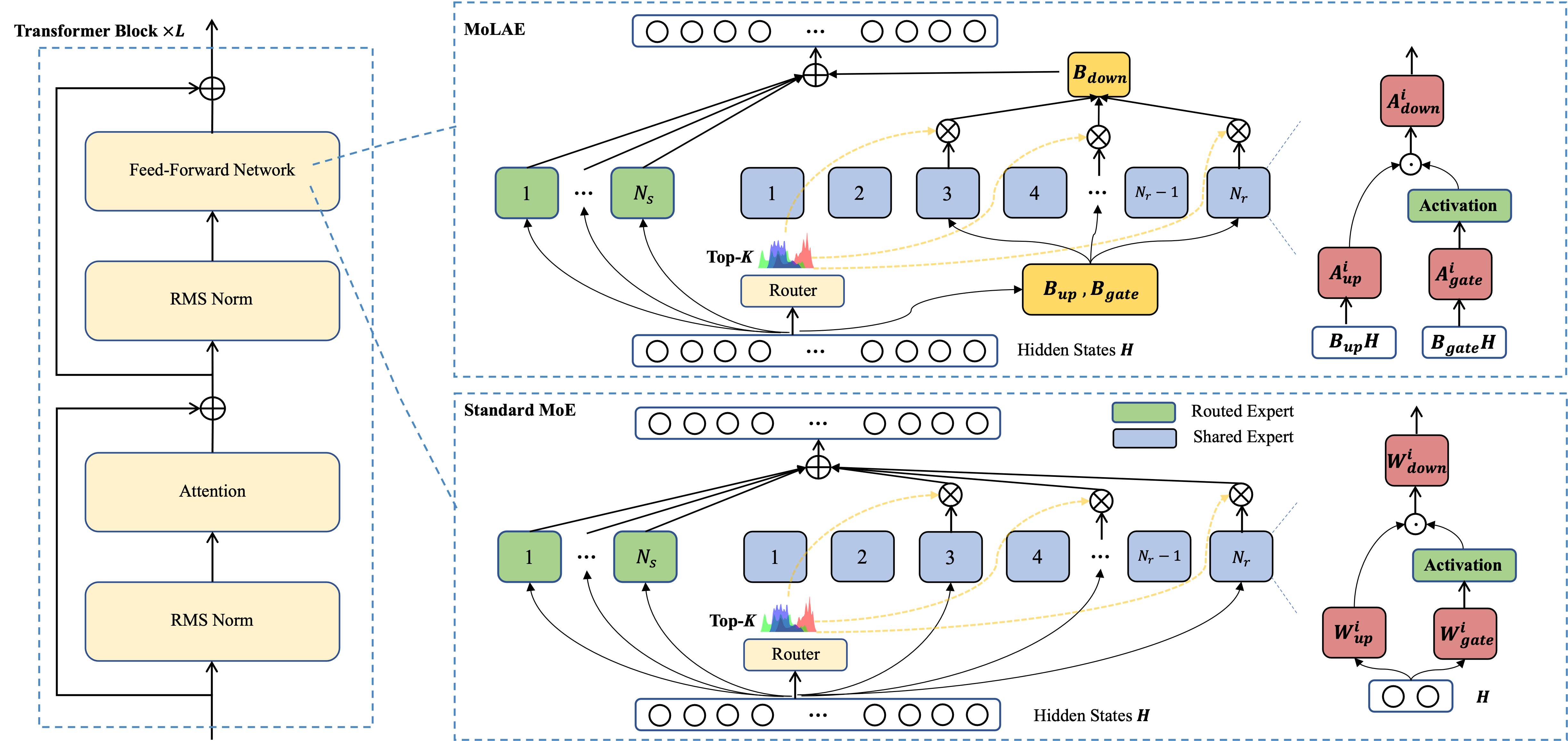} 
\caption{Architectural comparison between MoE and MoLAE in the FFN layer. In both diagrams, $N_r$ denotes the number of routed experts. MoLAE extends the conventional MoE architecture by introducing latent mappings $B_{\text{up}}$, $B_{\text{gate}}$, and $B_{\text{down}}$ that capture shared information across experts. Expert-specific information is encapsulated in the mappings $A_{\text{up}}^i$, $A_{\text{down}}^i$, and $A_{\text{gate}}^i$ for each expert $i$.}
\label{fig:1}
\vspace{-1em}
\end{figure}

In this work, we introduce \textbf{M}ixture \textbf{o}f \textbf{LA}tent \textbf{E}xperts (\textbf{MoLAE}), a novel parameterization paradigm that addresses the core inefficiencies of traditional MoE architectures. Our key insight is that expert modules in standard MoE models contain significant redundancy and operate in unnecessarily high-dimensional spaces. MoLAE reformulates each expert operation through a mathematically principled two-phase transformation: (1) a shared projection into a compressed latent space, followed by (2) expert-specific transformations within this lower-dimensional manifold.

Formally, MoLAE implements this insight by factorizing each expert's weight matrix $W^i \in \mathbb{R}^{m \times n}$ into the product $A^i B$, where $A^i \in \mathbb{R}^{m \times m}$ represents expert-specific transformations and $B \in \mathbb{R}^{m \times n}$ represents a shared projection into a latent space of dimension. This factorization yields substantial parameter reduction, particularly in contemporary LLM architectures where the hidden dimension $n$ significantly exceeds the MoE intermediate dimension $m$.

Our contributions are as follows:
1) We propose MoLAE, a parameter-efficient architecture that achieves competitive performance with standard MoE models while requiring significantly fewer parameters and reduced computational overhead.
2) We develop a theoretically grounded framework for transforming pre-trained MoE models into the MoLAE architecture, including a mathematical characterization of optimal factorization conditions and an efficient two-stage algorithm incorporating low-rank approximation techniques.
3) Through comprehensive empirical evaluation on multiple benchmark datasets, we demonstrate that MoLAE preserves or enhances model capabilities across diverse language tasks while substantially improving parameter efficiency, thereby enabling more economical scaling of large language models.

\section{Related Works} \label{sec:related}

\textbf{Finer-Grained Mixture of Experts}.
Mixture of Experts (MoE), initially introduced by \citet{jacobs_1991_adaptive} and \citet{jordan_1994_hierarchical}, has garnered significant attention in recent years \citep{aljundi_2017_expert, collobert_2001_parallel, deisenroth_2015_distributed, eigen_2013_learning, rasmussen_2001_infinite, shahbaba_2009_nonlinear, theis_2015_generative}. \citet{lepikhin_2020_gshard} pioneered the integration of MoE technology into transformer architectures, enabling substantial parameter scaling while maintaining computational efficiency. 
Subsequently, numerous studies have advanced MoE algorithms, particularly focusing on replacing feed-forward network (FFN) layers with MoE layers in transformer-based Large Language Models (LLMs) \citep{dai_2024_deepseekmoe, du_2022_glam, fedus_2022_switch, xue_2024_openmoe, zoph_2022_st}.

However, conventional GShard models exhibit limitations in capturing domain-specific expertise due to their relatively small number of experts. To address this constraint and enhance expert specialization, finer-grained MoE architectures were proposed by \citet{dai_2024_deepseekmoe} and subsequently implemented in several state-of-the-art models \citep{guo_2025_deepseek, liu_2024_deepseek, qwen_2024_qwenmoe}. In contrast to traditional GShard MoE designs, finer-grained variants incorporate substantially more experts, each with reduced parameter counts, enabling greater specialization in domain-specific knowledge representation and processing. This approach not only refines the decomposition of knowledge across experts, facilitating more precise learning, but also enhances the flexibility of expert activation combinations, allowing for more specialized and targeted knowledge capture.

\textbf{Algorithmic Design of MoE}.
The introduction of expert modules in LLMs introduces several algorithmic challenges that must be addressed during both training and inference phases. A critical aspect of MoE design is the gating function, which orchestrates the engagement of expert computations and the combination of their respective outputs. The gating mechanisms can be broadly categorized into three types: sparse, which activates a subset of experts; dense, which activates all experts; and soft, which encompasses fully-differentiable approaches including input token merging and expert merging \citep{pan_2024_dsmoe, zadouri_2023_mov, puigcerver_2022_softmoe}.

Sparse token-choice gating, where the gating function selects top-k experts for each input token, is the most prevalent approach \citep{fedus_2022_switch, lepikhin_2020_gshard, zoph_2022_st}. This method is often augmented with auxiliary loss functions to promote balanced expert utilization \citep{lepikhin_2020_gshard, fedus_2022_switch, du_2022_glam}. Alternative approaches include expert-choice gating, where each expert selects the top-k tokens they will process \citep{zhou_2022_expertchoice, zhou_2023_brainformers}, and non-trainable gating mechanisms that use predetermined routing strategies \citep{roller_2021_hash, costa_2022_thor, gururangan_2021_demix}.

A promising recent development in MoE is the integration with parameter-efficient fine-tuning (PEFT) techniques, creating Mixture of Parameter-Efficient Experts (MoPEs) \citep{zhang_2021_adammix, wu_2022_loramoe, ye_2023_moela}. These approaches combine the task versatility of MoE with the resource efficiency of PEFT, positioning them as a significant advancement in efficient multi-task learning.

% \textbf{Finer-Grained MoE}.
% Mixture of Expert, which firstly introduced in \citep{jacobs_1991_adaptive,jordan_1994_hierarchical}, has achieved huge attention since recent years \citep{aljundi_2017_expert, collobert_2001_parallel,deisenroth_2015_distributed,eigen_2013_learning,rasmussen_2001_infinite,shahbaba_2009_nonlinear,theis_2015_generative}.
% \citet{lepikhin_2020_gshard} applied the MoE technology to the transformer architecture, allowing for a significant increase in the number of parameters while maintain a low computation cost.
% Since them, many works delved into MoE algorithmic advancements, particularly the prevalent substitution of feed-forward network (FFN) layers with MoE layers in transformer-based LLMs \citep{dai_2024_deepseekmoe,du_2022_glam,fedus_2022_switch,xue_2024_openmoe,zoph_2022_st}.
% However, the GShard models are limited to capture the ``specific expert'' since the number of experts is small. 
% Thus, to improve the expert specialization, the finer-grained MoE is proposed in \citep{dai_2024_deepseekmoe} and them widely applied in \citep{guo_2025_deepseek,liu_2024_deepseek, qwen_2024_qwenmoe}
% Compared with the GShard MoE, the finer-grained version contains much more experts while the size of each expert becomes much smaller and specialized to the special knowledge.

% \textbf{Algorithm Design of MoE}.
% The introduction of the experts in the LLMs, leads to several problem during its training and inference tasks.
\section{Redundancy in Standard MoE models} \label{sec:bg}

\subsection{Background: Standard MoE Architecture}
Finer-grained MoE architectures for FFNs employ $N$ distinct experts \citep{dai_2024_deepseekmoe}. For each expert $E_i(x)$ where $i \in \{1, 2, \ldots, N\}$, the computation is defined as:
\begin{equation} \label{equ:expert}
E_i(x) := W^i_{\text{down}} \left( W^i_{\text{up}}(x) \odot \textsc{Act} \left( W_{\text{gate}}(x) \right) \right),
\end{equation}
where $\textsc{Act}$ represents the activation function, $W_{\text{up}}, W_{\text{gate}} \in \mathbb{R}^{m \times n}$, and $W_{\text{down}} \in \mathbb{R}^{n \times m}$ are linear operators. In this context, $n$ denotes the hidden dimension and $m$ represents the MoE intermediate dimension, where typically $m \leq n$.
For input $x$, the FFN layer output is computed as:
\begin{equation} \label{equ:ffn}
y = x + \sum_{i=1}^N g_i(x) E_i(x), 
\end{equation}
where $g_i: \mathbb{R}^n \rightarrow \mathbb{R}$ is the router function that determines the contribution of each expert.

While empirical evidence demonstrates that increasing the number of experts leads to superior performance across various applications, this approach introduces significant challenges. The proliferation of parameters results in substantially increased storage requirements and all-to-all network communication overhead, limiting scalability and efficiency.

\subsection{Parameter Redundancy in MoE Models} \label{sec:appendix1}
In this section, we provide empirical evidence for significant parameter redundancy within FFN layers, substantiating the theoretical framework presented in Section \ref{sec:theory}. We conduct analysis using the Qwen1.5-MoE-A2.7B \citep{qwen_2024_qwenmoe} model, a popular MoE model comprising 14.3B parameters while activating only 2.7B parameters during inference, with 60 distinct experts.

For our analysis, we define a ratio-$r$ low-rank approximation $\tilde{W}$ of any matrix $W$ as a matrix whose rank equals $r \times \text{rank}(W)$, where $0 < r \leq 1$. In accordance with the Eckart-Young-Mirsky theorem \citep{schmidt_1989_theorie}, these approximations are computed via Singular Value Decomposition (SVD), retaining only the largest $r \times \text{rank}(W)$ singular values and their corresponding singular vectors.
To rigorously assess model capabilities under low-rank constraints, we evaluate performance on three benchmark tasks: MMLU \citep{hendrycks_2021_measuring}, GSM8K \citep{cobbe_2021_gsm8k}, and Wikitext-2 \citep{merity_2016_pointer}. All experiments were conducted using the lm-evaluation-harness evaluation framework \citep{gao_2024_framework}.
Table \ref{table:4} presents the comparative results.

\begin{table}[htbp]
\begin{center}
\caption{Performance comparison of MoE models with varying low-rank approximation ratios across multiple benchmarks.}
\begin{tabular}{c | c c c}
\toprule
\textbf{Low-rank ratio $r$} & \textbf{GSM8K} (\%) $\uparrow$ & \textbf{MMLU} (\%) $\uparrow$ & \textbf{Wikitext PPL} $\downarrow$ \\ 
\midrule
$1.0$ (Original) & 60.1 & \textbf{61.0} & \textbf{9.49} \\
$0.8$ & \textbf{61.2} & 60.3 & 9.65 \\
$0.6$ & 60.1 & 59.4 & 9.85 \\
\bottomrule
\end{tabular}
\label{table:4}
\end{center}
\end{table}

Table \ref{table:4} illustrates the relationship between rank reduction and model performance. The baseline case ($r = 1.0$) represents the original, unmodified model with full-rank weight matrices. Notably, when reducing the rank of FFN operators by 20\% ($r = 0.8$), we observe no significant performance degradation. In fact, the reduced-rank model demonstrates superior performance on the GSM8K benchmark, achieving a 1.1 percentage point improvement over the original model, while maintaining comparable performance on MMLU accuracy and Wikitext-2 PPL.
These empirical findings provide compelling evidence that, despite the mathematical full-rank property of FFN weight matrices, a substantial proportion of parameters contain redundant information that can be effectively approximated through lower-dimensional representations. This parameter redundancy phenomenon forms the empirical foundation for our theoretical analysis in Section \ref{sec:theory}.
\section{MoLAE: Mixture of Latent Experts}
\label{sec:mole_structure}

In this section, we introduce the Mixture of Latent Experts (MoLAE), a novel framework that maps experts into latent space to address several limitations of traditional MoE models. For clarity and focus, we exclude shared-experts from our analysis throughout this paper.

\subsection{Mixture of Latent Experts: Concept and Design}
To address these limitations, we propose MoLAE framework, which fundamentally redefines the structure of FFN layers in expert-based systems.
Our approach is informed by a careful analysis of the expert computation in Equation \eqref{equ:expert}, which comprises three distinct operations:
\begin{enumerate}
\item \textbf{Projection in:} The input $x$ is mapped from the high-dimensional space $\mathbb{R}^n$ to a lower-dimensional space $\mathbb{R}^m$ via linear operators $W_{\text{up}}$ and $W_{\text{gate}}$.
\item \textbf{Non-linear transformation:} A one-layer neural network applies a non-linear transformation through the Hadamard product and activation function.
\item \textbf{Projection out:} The intermediate output is mapped back from the low-dimensional space to the original high-dimensional space.
\end{enumerate}

A critical insight is that the core functionality of experts primarily stems from the non-linear transformation in the second step. The projection operations in the first and third steps primarily serve to reduce computational complexity, potentially at the cost of limiting the expert's domain capacity.

Drawing inspiration from both Multi-Head Latent Attention (MLA) \citep{liu_2024_deepseek}, which introduces a ``latent space'' for KV caches in attention layers, and Grouped-Query Attention (GQA) \citep{ainslie_2023_gqa}, which leverages group-based processing, we propose the MoLAE, which operates on experts, i.e. specific FFN layers in standard MoE models.
This architecture fundamentally reconsiders how inputs are projected into a lower-dimensional latent space, enabling more efficient computation within the experts.

\subsection{Formulation of Mixture of Latent Experts}

To formalize the concept of latent experts, we examine the decomposition of expert-specific operators through matrix factorization. Using the ``up operator'' $W_{\text{up}}^i$ of expert $i$ as a representative example, we propose a structured factorization where:
\begin{equation}
W_{\text{up}}^i = A_{\text{up}}^i B_{\text{up}}.
\end{equation}
In this formulation, $B_{\text{up}} \in \mathbb{R}^{m \times n}$ functions as a unified projection operator shared across experts, mapping inputs from the high-dimensional space $\mathbb{R}^n$ to a lower-dimensional latent space $\mathbb{R}^m$, where typically $m \ll n$. Conversely, $A_{\text{up}}^i \in \mathbb{R}^{m \times m}$ represents an expert-specific linear transformation within this latent space, encapsulating the specialized function of expert $i$.
Following the terminology established in MLA, we designate $B_{\text{up}}$ as the latent mapping for the ``up operator''. 
This factorization principle extends systematically to the ``gate operator'' as well.
On the other side, for the ``down operator'' $W_{\text{down}}^i$, which maps from a lower-dimensional to a higher-dimensional space, the decomposition necessarily assumes a reverse form: $W_{\text{down}}^i = B_{\text{down}} A_{\text{down}}^i$, where $A_{\text{down}}^i \in \mathbb{R}^{m \times m}$ and $B_{\text{down}} \in \mathbb{R}^{n \times m}$.

To optimize the trade-off between model expressivity and parameter efficiency, we introduce a structured grouping mechanism where each subset of $k$ experts shares the same latent mapping matrices $B_{\text{up}}$ and $B_{\text{down}}$. This design establishes a configurable spectrum of architectural possibilities: when $k=1$, each expert maintains its independent latent space, and MoLAE becomes functionally equivalent to the standard MoE architecture. Conversely, as $k$ increases, the model achieves progressively higher parameter efficiency at a measured trade-off with expert specialization. This parameterization allows for systematic exploration of the efficiency-performance frontier in mixture-of-experts architectures.

% \subsection{Formal Definition of MoLAE}
We now provide a formal definition of the MoLAE architecture. Let $\lfloor \cdot \rfloor$ denote the floor function and $\{(A_{\text{up}}^i, A_{\text{gate}}^i, A_{\down}^i) \}_{i=1}^{N}$ represent the set of expert-specific latent transformations. The $i$-th expert is defined as:
\begin{equation} \label{equ:latent_expert}
E_i(x) = B_{\down}^{\lfloor  i / k\rfloor} A_{\down}^i \left( A_{\text{up}}^i B_{\text{up}}^{\lfloor i/k \rfloor}(x) \odot \textsc{Act}\left(A_{\text{gate}}^i B_{\text{gate}}^{\lfloor i/k \rfloor}(x)\right) \right),
\end{equation}
where $k$ is the group size of experts.
Consequently, the output of the FFN layer employing the MoLAE architecture is computed as:
\begin{equation}
y = x + \sum_{i=1}^N g_i(x) E_i(x).
\end{equation}

This formulation effectively disentangles expert-specific computations from the shared dimensionality reduction operations, enabling significant parameter efficiency while preserving model expressivity.
The visual comparison between MoE and MoLAE architectures is shown in Figure \ref{fig:1}.

\subsection{Efficiency Benefits of MoLAE}

The MoLAE architecture offers significant efficiency advantages over standard MoE models, spanning multiple computational dimensions from memory usage to communication overhead. To quantify these benefits, we provide a comparative analysis between MoE and MoLAE for a single FFN layer, assuming identical configurations with hidden dimension $n$, MoE intermediate dimension $m$, and number of experts $N$.

\begin{table}[htbp]
\begin{center}
\caption{Efficiency comparison between standard MoE and our proposed MoLAE architectures for a single FFN layer.}
\begin{tabular}{c c c}
\toprule
\textbf{Architecture} & \textbf{Parameter Count} & \textbf{FLOPs per Forward Pass} \\
\midrule 
Standard MoE & $3Nmn$ & $(3mn + 2m)N$ \\
\midrule
MoLAE (Ours) & $3Nm^2 + 3 \lfloor \frac{N}{k} \rfloor mn$ & $\left(3m^2 + 2m\right)N + 3 \lfloor \frac{N}{k} \rfloor mn$ \\
\bottomrule
\end{tabular}
\label{table:1}
\end{center}
\end{table}

\textbf{Parameter Efficiency}
As shown in Table \ref{table:1}, MoLAE substantially reduces the parameter count compared to standard MoE, particularly when $m \ll n$, which is the typical case in modern LLMs. For instance, in DeepSeek-V3, $n = 7168$ while $m = 2048$. The parameter reduction stems from our latent parameterization, where expert-specific transformations operate in the lower-dimensional latent space ($m \times m$) rather than directly on the high-dimensional hidden space ($m \times n$).

\textbf{Computational Efficiency}
Beyond parameter savings, MoLAE also reduces the computational cost measured in FLOPs. The efficiency gain becomes particularly pronounced when the number of experts $N$ is large and $k$ is small (meaning fewer latent projection matrices are used). This computational advantage translates to faster inference and training times, especially on hardware where memory bandwidth is a bottleneck.

\textbf{Communication Overhead Reduction}
A critical but often overlooked benefit of MoLAE is the reduction in all-to-all communication costs during distributed training and inference. In standard MoE models, the full expert parameters ($3Nmn$ in total) must be synchronized across devices. In contrast, MoLAE requires synchronization of significantly fewer parameters, reducing network bandwidth requirements and improving scalability for distributed deployments.

\textbf{Memory Access Patterns}
MoLAE also offers improved cache efficiency during computation. The smaller matrices used in latent transformations ($A^i \in \mathbb{R}^{m \times m}$) exhibit better locality of reference compared to the larger matrices in standard MoE ($W^i \in \mathbb{R}^{m \times n}$), potentially leading to higher utilization of fast cache memory and reduced main memory bandwidth demands.

\section{Transformation from MoE to MoLAE} \label{sec:theory}
In this section, we establish the theoretical foundation for transforming a standard MoE model into its corresponding MoLAE counterpart. We focus on the ``up operator'' as a representative example. For analytical clarity, we make two simplifying assumptions: (1) we omit the subscript of $W_{\text{up}}^i$ for ease of notation during this analysis, and (2) we consider the case where $k=N$, implying a single shared latent space operator for all experts.

\textbf{Keep the ``up operator''}.
% Experimental results from the MoLAE transformation of Qwen1.5-MoE-A2.7B suggest that the transference of all operators does not detrimentally affect the performance of the resultant MoLAE model.
% However, empirical investigations detailed in Appendix \ref{sec:appendix3} reveal that for the Moonlight-16B-A3B model, preserving the architectural integrity of the ``up operator'' yields significantly superior performance compared to its alteration.
% Consequently, within the transformed MoLAE architecture for the Moonlight model, the ``up operator'' is maintained, while only the ``down operator'' and ``gate operator'' are converted from the MoE structure to the MoLAE configuration.
% In contrast, all operators are transferred for the Qwen model.
Experiments on Qwen1.5-MoE-A2.7B indicate that transferring all operators during MoLAE transformation does not hurt performance.
However, for Moonlight-16B-A3B, empirical results (Appendix \ref{sec:appendix3}) show superior performance when preserving the ``up operator'' structure.
Thus, for Moonlight, only its ``down'' and ``gate'' operators are converted to the MoLAE structure, retaining the ``up operator''.
Conversely, all operators are transferred for Qwen models.

\subsection{Transformation via Matrix Factorization} \label{subsec:trans}

For a given weight matrix $W^i$ associated with expert $i$, we aim to find corresponding matrices $A^i$ and $B$ such that $W^i X \approx A^i B X$ for the activation $X$. This transformation represents the post-training perspective where we seek to convert pre-trained MoE parameters into the MoLAE architecture.

One direct approach is to determine matrices $A^i$ and $B$ such that $A^i B \approx W^i$ for all $i \in \{1, 2, \ldots, N\}$. This leads naturally to the following optimization problem:
\begin{equation} \label{equ:5}
\min_{A^i, B} \quad F(A^1, \cdots, A^N, B) := \frac{1}{2} \sum_{i=1}^N \| W^i - A^i B \|_F^2,
\end{equation}
where $\| \cdot \|_F$ denotes the Frobenius norm. 
Problem \eqref{equ:5} admits the theoretical optimal solutions by using the SVD decomposition technique.

\textbf{Closed-Form Solution via Singular Value Decomposition}.
Obviously, problem \eqref{equ:5} has infinitely many solutions since one can change one optimal solution $(A^*, B^*)$ into another by selecting a constant $\lambda$ and obtain another optimal solution $(\lambda A^*, \frac{1}{\lambda} B^*)$.
Hence, we only provide ``one'' optimal solution here.

To derive a closed-form solution to problem \eqref{equ:5}, we consolidate the matrices into the following block structures:
\begin{equation}
W = \begin{pmatrix}
W^1 \\ 
W^2 \\ 
\vdots \\
W^N
\end{pmatrix}, \quad 
A = \begin{pmatrix}
A^1 \\ 
A^2 \\ 
\vdots \\
A^N
\end{pmatrix}.
\end{equation}
With this notation, problem \eqref{equ:5} can be reformulated as:
\begin{equation} \label{equ:6}
\min_{A, B} \quad \frac{1}{2} \| W - AB \|_F^2,
\end{equation}
where $A \in \mathbb{R}^{mN \times m}$ and $B \in \mathbb{R}^{m \times n}$.

While problem \eqref{equ:6} has infinitely many solutions due to its underdetermined nature, the Eckart-Young-Mirsky theorem \citep{schmidt_1989_theorie} provides an optimal solution with respect to the Frobenius norm. Specifically, let $W = U \Sigma V^\top$ be the singular value decomposition (SVD) of $W$, where:
\begin{itemize}
\item $U \in \mathbb{R}^{mN \times mN}$ is an orthogonal matrix whose columns are the left singular vectors of $W$
\item $\Sigma \in \mathbb{R}^{mN \times n}$ is a rectangular diagonal matrix with singular values $\sigma_1 \geq \sigma_2 \geq \cdots \geq \sigma_{\min(n, mN)} \geq 0$ on its diagonal
\item $V \in \mathbb{R}^{n \times n}$ is an orthogonal matrix whose columns are the right singular vectors of $W$
\end{itemize}
Let $\Sigma_m$ be the truncated version of $\Sigma$ that retains only the $m$ largest singular values: $\Sigma_m = \text{diag}(\sigma_1, \sigma_2, \ldots, \sigma_m, 0, \ldots, 0) \in \mathbb{R}^{mN \times n}$.
According to the Eckart-Young-Mirsky theorem, an optimal solution to problem \eqref{equ:6} is given by:
\begin{equation}
A^* = U_m \Sigma_m^{1/2}, \quad B^* = \Sigma_m^{1/2} V_m^\top,
\end{equation}
where $U_m \in \mathbb{R}^{mN \times m}$ consists of the first $m$ columns of $U$, $\Sigma_m^{1/2} \in \mathbb{R}^{m \times m}$ is a diagonal matrix containing the square roots of the $m$ largest singular values, and $V_m \in \mathbb{R}^{n \times m}$ consists of the first $m$ columns of $V$.
This factorization yields the minimum Frobenius norm error among all rank-$m$ approximations of $W$, with the approximation error given by:
\begin{equation}
\|W - A^*B^*\|_F^2 = \sum_{i=m+1}^{\min(n,mN)} \sigma_i^2.
\end{equation}

To ensure the effectiveness of this transformation, we provide more details in Appendix \ref{subsec:residual} which discuss how to minimize the factorization residuals during the transformation.

\subsection{Transfer MoE to MoLAE: A Unified Framework}

Based on our theoretical analyses above, we propose a unified, systematic framework for transforming Mixture of Experts (MoE) models into their more parameter-efficient Mixture of Latent Experts (MoLAE) counterparts. Our framework consists of two principal steps, carefully designed to preserve model capabilities while enabling the latent parameterization.

\textbf{Step 1: Rank Reduction.} For each expert operator $W^i$, we compute a low-rank approximation $\tilde{W}^i$ that maintains the essential functionality of the original operator while increasing the dimension of its nullspace. This step is motivated by our theoretical analysis showing that larger nullspace intersections facilitate better factorization. We determine the optimal rank based on empirical validation to ensure minimal performance degradation.

\textbf{Step 2: Matrix Factorization.} Using the rank-reduced operators $\{\tilde{W}^i\}_{i=1}^N$, we apply matrix factorization techniques to identify the shared projection matrix $B$ and the expert-specific latent transformations $\{A^i\}_{i=1}^N$. For this step, we employ the SVD-based approach detailed in Section \ref{subsec:trans}, which provides the optimal factorization with respect to the Frobenius norm.

We formalize our approach in Algorithm \ref{alg:moe_to_mole}, which provides a complete computational procedure for transforming MoE parameters into the MoLAE architecture.

\begin{algorithm}
\small
\caption{Transformation of MoE to MoLAE}
\label{alg:moe_to_mole}
\begin{algorithmic}[1]
\Require Expert weight matrices $\{W^i\}_{i=1}^N$, target rank $r$, latent dimension $m$
\Ensure Latent expert matrices $\{A^i\}_{i=1}^N$, shared projection matrix $B$

\State // Step 1: Rank Reduction
\For{$i = 1$ to $N$}
    \State Compute SVD: $W^i = U^i \Sigma^i (V^i)^\top$
    \State Truncate to rank $r$: $\tilde{W}^i = U^i[:,:r] \cdot \Sigma^i[:r,:r] \cdot (V^i[:,:r])^\top$
\EndFor

\State // Step 2: Matrix Factorization
\State Construct concatenated matrix: $\tilde{W} = [\tilde{W}^1; \tilde{W}^2; \ldots; \tilde{W}^N]$
\State Compute SVD: $\tilde{W} = U \Sigma V^\top$
\State Extract first $m$ singular values and vectors
\State $A = U[:,:m] \cdot \Sigma[:m,:m]^{1/2}$
\State $B = \Sigma[:m,:m]^{1/2} \cdot V[:,:m]^\top$
\State Partition $A$ into $N$ blocks to obtain $\{A^i\}_{i=1}^N$

\State \Return $\{A^i\}_{i=1}^N, B$
\end{algorithmic}
\end{algorithm}

This algorithm provides a computationally efficient procedure for transforming standard MoE layers into MoLAE architecture. The rank reduction parameter $r$ and the latent dimension $m$ serve as hyperparameters that can be tuned to balance performance preservation against parameter efficiency. Typically, we set $r \leq m < \min(n, \sum_{i=1}^N m)$, where $n$ is the input dimension.

Our framework is applicable to all linear operators within an MoE layer, including the up, down, and gate operators. By applying this transformation to each operator independently, we can convert the MoE model to MoLAE while maintaining overall performance.
\section{Experiments} \label{sec:exp}
In this section, we evaluate the effectiveness of MoLAE on downstream tasks and the pre-training performance on GPT-2 \citep{radford_2019_languageMA}. More experiments, like the importance of different ``up, gate, down'' operators and model configurations, are provided in Appendix \ref{sec:appendix3}.

\begin{table}[t!]
\caption{Comparisons of Qwen1.5-MoE-A2.7B and Moonlight-16B-A3B to their MoLAE architectures. To simplify the computations of efficiency, we only consider the MoE parts of the model.}
\centering
\fontsize{10}{11} \selectfont
\def\arraystretch{1,2}
\begin{tabular}{ccccccc}
\toprule
\multirow{2}{*}{Models} & \multicolumn{5}{c}{Benchmarks} & \multirow{2}{*}{Params} \\
\cline{2-6}
& MMLU $\uparrow$ & GSM8K $\uparrow$ & CEval $\uparrow$& MNLI $\uparrow$ & Wiki PPL $\downarrow$\\ 
\midrule
Qwen1.5-MoE & 61.0 & 60.1 & 80.2 & 49.9 & 9.49 & 14.3B \\
\tabincell{c}{Qwen1.5-MoLAE\\($k$=10)} & 60.0 & 59.4 & 74.1 & 40.6 & 9.74 & 10.5B \\
\hdashline
Moonlight-MoE & 67.2 & 77.4 & 74.2 & 43.7 & 7.12 & 16.0B\\
\tabincell{c}{Moonlight-MoLAE\\($k$=8)} & 60.2 & 72.3 & 62.8 & 38.7 & 10.6 & 12.3B\\
\bottomrule
\end{tabular}
\label{table:5}
\vspace{-1em}
\end{table}

\subsection{Transformation from MoE to MoLAE: Downstream Tasks} \label{sec:appendix2}
We first present our empirical analysis of transforming standard MoE architectures into their corresponding MoLAE counterparts. We specifically investigate two popular MoE models, including Qwen1.5-MoE-A2.7B model \cite{qwen_2024_qwenmoe} and Moonlight-16B-A3B \citep{liu_2025_muon}, on diverse tasks, such as MMLU \citep{hendrycks_2021_measuring}, GSM8K \citep{cobbe_2021_gsm8k}, CEval \citep{huang_2023_ceval}, MNLI \citep{wang2019glue} and Wikitext-2 \citep{merity_2016_pointer}.

As shown in Table \ref{table:5}, MoLAE achieves performance comparable to that of the original MoE models, while exhibiting notable parameter efficiency and only slight performance degradations across the benchmarks. This result highlights the effectiveness of the transformation from MoE to MoLAE architectures. Furthermore, as indicated in Section \ref{sec:mole_structure}, the group size $k$ in MoLAE architecture represents a critical hyperparameter that controls the trade-off between parameter efficiency and model expressivity. We evaluate several distinct configurations of \( k \in \{1, 10, 20, 30, 60\} \) on the Qwen1.5-MoE model, which consists of 60 experts. Specifically, \( k=1 \) corresponds to the original MoE architecture; \( k=10 \) and \( k=20 \) represent balanced MoLAE configurations with multiple latent spaces; while \( k=30 \) and \( k=60 \) denote extreme cases with only one or two shared latent spaces, respectively.
As shown in Table \ref{fig:ablation_k}, with a moderate group size, such as \( k=10 \) for the Qwen1.5-MoE model, MoLAE largely preserves performance across various benchmarks. However, as the group size increases, the performance of the MoLAE model progressively deteriorates. For simpler tasks, such as CEval and MNLI, the performance decline is minimal, indicating that MoLAE retains most of the model's capacity. In contrast, for more challenging tasks, such as GSM8K, there is a significant performance degradation. This dramatic decline empirically supports our theoretical analysis, which suggests that multiple latent spaces are essential to maintain the capacity of the original MoE model.

\begin{figure}[t!]
  % Figure 1
  \begin{minipage}{0.45\textwidth}
    \includegraphics[width=\textwidth]{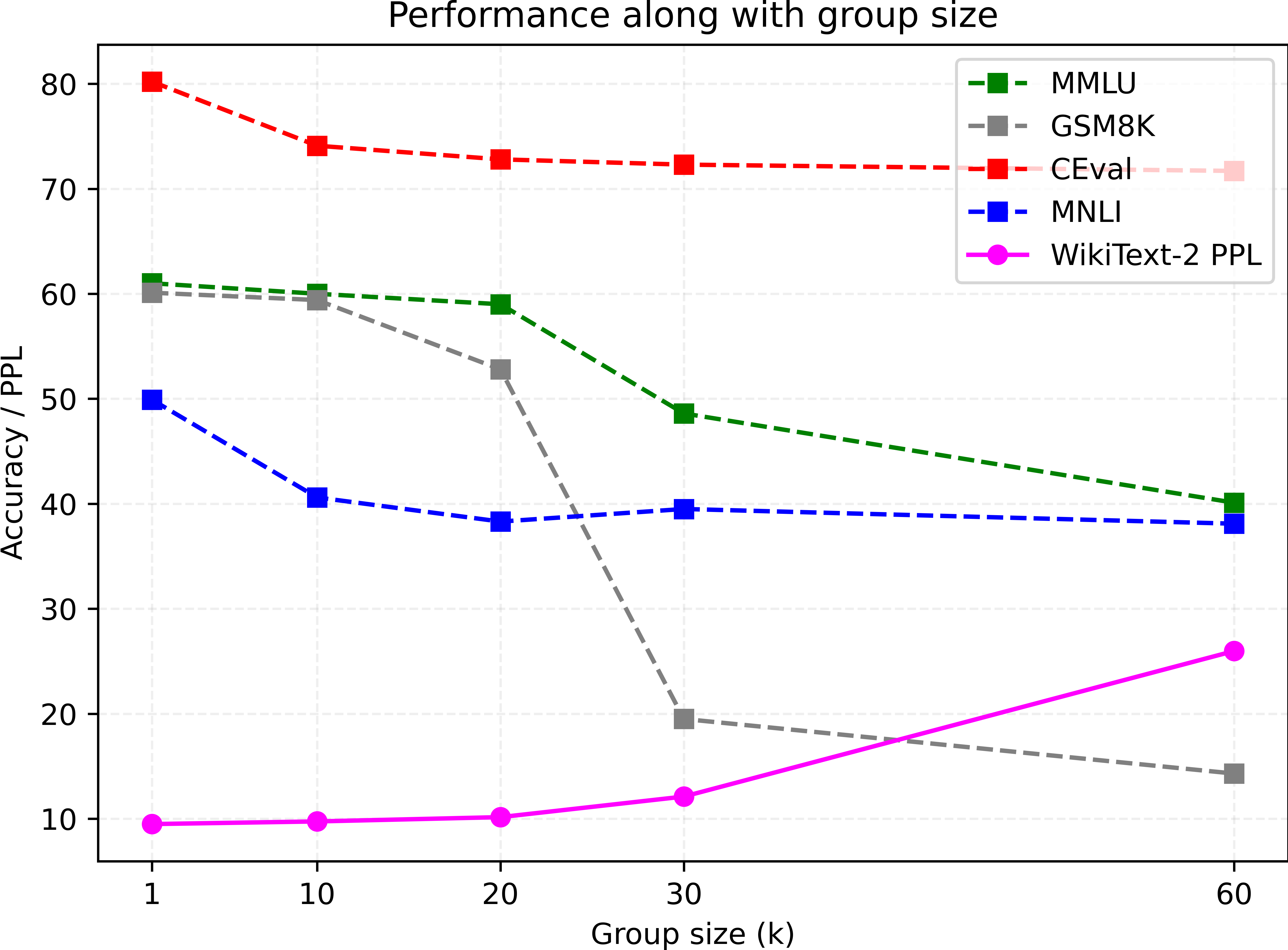}
    \captionof{figure}{Ablation study of group size $k$ on the Qwen1.5-MoE model.}
    \label{fig:ablation_k}
  \end{minipage}
  \hfill
  % Figure 2
  \begin{minipage}{0.5\textwidth}
    \includegraphics[width=\textwidth]{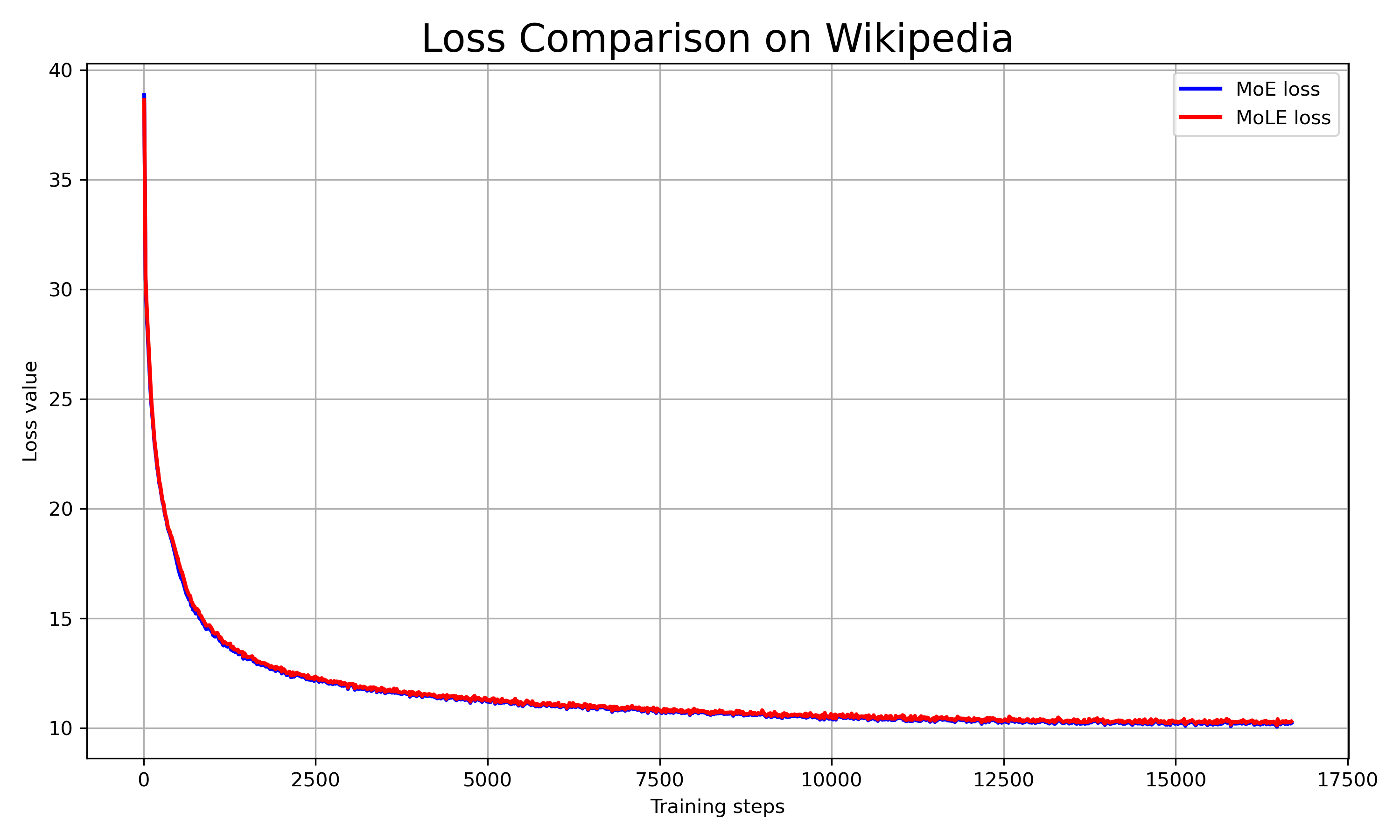}
    \captionof{figure}{Comparison of training loss curves between MoE and MoLAE models on the English Wikipedia dataset.}
    \label{fig:comparison}
  \end{minipage}
  \vspace{-1em}
\end{figure}

\subsection{Pretraining of MoLAE}
To further validate the effectiveness of MoLAE, we construct the paired MoE and MoLAE models derived from the GPT-2 model \citep{radford_2019_languageMA}, and then conduct the pretraining on the models.
The MoE (151M) and MoLAE (94M) models are configured with identical architectural parameters except for FFN layer structures, as detailed in Table \ref{table:2}. Both architectures implement $N = 32$ experts in their respective FFN layers. For MoLAE model, we establish $k = 8$, indicating that each group of $8$ experts shares a common latent representation space. All models are trained on the Wikipedia English \citep{wikidump} with maximum length of $512$. Models are updated using AdamW optimizer with consistent hyper-parameters across all runs.

\paragraph{Parameter Efficiency}
As quantified in Table \ref{table:2}, the introduction of shared latent spaces in MoLAE architecture yields a substantial reduction in model parameter count compared to the standard MoE architecture. Specifically, MoLAE achieves a 40\% reduction in non-embedding parameters while maintaining comparable model capacity. This parameter efficiency represents a significant advancement in model scalability without sacrificing performance.

\paragraph{Training Dynamics}
Figure \ref{fig:comparison} illustrates the pretraining convergence characteristics of both MoE and MoLAE models. The training loss trajectories reveal that MoLAE maintains competitive optimization dynamics despite its significantly reduced parameter count. Although the MoE model converges to marginally lower loss values, this difference is negligible when considering the substantial parameter efficiency gained with MoLAE. These results suggest that the shared latent representation in MoLAE effectively preserves the essential modeling capacity while eliminating redundant parameterization inherent in traditional MoE architectures. Using the well-trained MoE and MoLAE models, we evaluated their performance on the downstream task of Wikitext-2 perplexity. The results indicate that the performance gap between the two models is minimal, with PPL values of 79.5 and 81.5, respectively. This suggests that the MoLAE architecture serves as an effective pretraining base model, offering excellent efficiency while maintaining competitive performance.

% \begin{figure}[t!]
% \centering
% \includegraphics[width=0.6\textwidth]{imgs/loss_comparison.png}
% \caption{Comparison of training loss curves between MoE and MoLAE models on the English Wikipedia dataset. The convergence patterns demonstrate comparable performance, with MoLAE exhibiting slightly higher loss values while utilizing substantially fewer parameters than the standard MoE architecture.}
% \label{fig:training_loss}
% \end{figure}

% \subsection{Training Dynamics}

% The pretraining convergence characteristics of both models are illustrated in Figure \ref{fig:1}. The training loss trajectories demonstrate that despite the significant parameter reduction, MoLAE exhibits competitive optimization dynamics compared to the standard MoE. While the MoE model achieves marginally lower loss values, the difference is negligible considering the substantial parameter efficiency gained with MoLAE. This indicates that the shared latent representation effectively preserves the essential modeling capacity while eliminating redundant parameterization common in traditional MoE architectures.

% \begin{figure}[t!]
% \centering
% \includegraphics[width=0.6\textwidth]{imgs/loss_comparison.png}
% \caption{The training loss of MoE and MoLAE in the Wikipedia English dataset. As shown in this image, the loss of MoE and MoLAE are quite similar, while the MoLAE is slightly larger than the MoE architecture.}
% \label{fig:2}
% \end{figure}

\section{Conclusion} \label{sec:conclusion}

% In this paper, we presented Mixture of Latent Experts (MoLAE), a novel parameter-efficient architecture that addresses the fundamental limitations of traditional MoE models. By factorizing expert weight matrices into shared projections and expert-specific transformations within a lower-dimensional latent space, MoLAE significantly reduces parameter count and computational overhead without sacrificing model performance. Our comprehensive empirical evaluation demonstrates that MoLAE not only preserves the capabilities of standard MoE models across diverse language tasks but also enables more economical scaling of large language models.

% The theoretical framework we developed for transforming pre-trained MoE models into the MoLAE architecture provides a principled approach to expert parameterization, offering insights into the nature of redundancy in neural networks. As models continue to grow in scale and complexity, architectures like MoLAE that prioritize parameter efficiency while maintaining performance will become increasingly valuable for practical deployment scenarios. Future work could explore extending these factorization techniques to other components of the transformer architecture and investigating dynamic adaptation of the latent space dimensionality based on input complexity.

We introduce Mixture of Latent Experts (MoLAE), which overcomes limitations of traditional MoE models by factorizing expert weights into shared projections and expert-specific transformations in a lower-dimensional space. This approach reduces parameters and computation while maintaining performance across language tasks. Our theoretical framework for converting pre-trained MoE models to MoLAE provides insights into neural network redundancy. As models grow, such parameter-efficient architectures become increasingly valuable. Future work could extend these techniques to other transformer components and explore dynamic latent space adaptation.

\bibliographystyle{unsrtnat}
\bibliography{reference.bib}

\newpage
\begin{appendices}

\section{Minimizing Factorization Residuals} \label{subsec:residual}

In the previous subsection \ref{sec:theory}, we established methods for transforming MoE models into their MoLAE counterparts from a theoretical perspective. Here, we address a critical aspect of this transformation: minimizing the residual error that inevitably arises when factorizing expert weights. We begin by establishing the precise conditions under which exact factorization is possible.

\begin{theorem} \label{thm:1}
Given matrices $W^i \in \mathbb{R}^{m \times n}$ with $m \leq n$, there exist matrices $A^i \in \mathbb{R}^{m \times m}$ and a common matrix $B \in \mathbb{R}^{m \times n}$ such that $A^i B = W^i$ for all $i \in \{1,2,\ldots,N\}$, \textit{if and only if} there exists an $(n - m)$-dimensional subspace $K \subseteq \mathbb{R}^n$ satisfying:
\begin{equation}
K \subseteq \bigcap_{i=1}^N \ker(W^i).
\end{equation}
\end{theorem}

\begin{proof}
\textbf{Necessity:} Suppose there exist $B \in \mathbb{R}^{m \times n}$ and $A^i \in \mathbb{R}^{m \times m}$ such that $W^i = A^i B$ for all $i$. Since $m \leq n$ and we require exact factorization, $B$ must be row-full-rank (i.e., $\text{rank}(B) = m$). Consequently, its right nullspace $\ker(B)$ has dimension $n - m$. For any vector $x \in \ker(B)$, we have:
\begin{equation}
W^i x = A^i B x = A^i \cdot 0 = 0,
\end{equation}
which implies $\ker(B) \subseteq \ker(W^i)$ for all $i$. Setting $K = \ker(B)$, we obtain an $(n - m)$-dimensional subspace contained in the intersection of all $\ker(W^i)$.

\textbf{Sufficiency:} Suppose there exists an $(n - m)$-dimensional subspace $K \subseteq \mathbb{R}^n$ such that $K \subseteq \ker(W^i)$ for all $i$. We can construct $B \in \mathbb{R}^{m \times n}$ such that $\ker(B) = K$. Since $\dim(K) = n - m$, the matrix $B$ has rank $m$ by the rank-nullity theorem. For each $W^i$, the inclusion $K \subseteq \ker(W^i)$ implies that any vector mapped to zero by $B$ is also mapped to zero by $W^i$. By the fundamental theorem of linear algebra, this means $\text{Row}(W^i) \subseteq \text{Row}(B)$, where $\text{Row}(\cdot)$ denotes the row space. Therefore, there exists $A^i \in \mathbb{R}^{m \times m}$ such that $W^i = A^i B$ for each $i$.
\end{proof}

Theorem \ref{thm:1} provides a geometric interpretation of the factorization problem: exact factorization is possible only when the nullspaces of all expert matrices share a sufficiently large common subspace. In practical LLM implementations, however, this condition is rarely satisfied for FFN layers in MoE models, as our empirical analysis confirms.

Given that exact factorization is generally unattainable, we now consider how to minimize the approximation error through strategic rank reduction. The rank-nullity theorem \citep{lang_1987_linear} states that for any linear mapping $W: X \rightarrow Y$:
\begin{equation}
    \text{rank}(W) + \dim(\ker(W)) = \dim(X).
\end{equation}
In our context, $X$ represents the hidden space ($\mathbb{R}^n$) and $Y$ the MoE intermediate space ($\mathbb{R}^m$). Therefore, $\text{rank}(W^i) + \dim(\ker(W^i)) = n$ for all $i \in \{1,2,\ldots,N\}$.

This relationship suggests a strategic approach: by reducing the rank of each $W^i$, we can increase the dimension of its nullspace. Specifically, if we constrain each $W^i$ to have rank $r < m$, then $\dim(\ker(W^i)) = n - r > n - m$. This increases the probability of finding a substantial common subspace within the intersection $\bigcap_{i=1}^N \ker(W^i)$, thereby improving the quality of our factorization.

We implement this approach by computing low-rank approximations of each $W^i$ before attempting factorization. Importantly, our empirical experiments in Appendix \ref{sec:appendix1} demonstrate that this rank reduction has minimal impact on model performance, suggesting that these FFN operators in MoE models inherently possess low-rank structure that can be exploited for more efficient parameterization.

\section{Critical Role of the ``Up Operator''} \label{sec:appendix3}

While empirical experiments on Qwen1.5-MoE-A2.7B demonstrate that transferring all operators for MoE models can achieve superior performance, we observe that different operators contribute differentially to MoLAE transformation efficacy. This section presents empirical evidence establishing that the ``up operator'' encapsulates more essential information than other components in MoE models. Through systematic experimentation, we demonstrate the importance of preserving this operator's structure for maintaining model performance.

We examine a different MoE architecture, Moonlight-16B-A3B \citep{liu_2025_muon}, in which the critical role of the ``up operator'' is more pronounced. Consistent with our methodology in Section \ref{sec:appendix2}, we utilize the Moonlight-16B-A3B model with a fixed latent parameter $k=8$, as it contains $64$ experts per layer. To isolate the significance of the ``up operator,'' we implement the following distinct transformation approaches:

\begin{enumerate}
\item \textbf{Partial transformation (``up+gate'')} - Converts only the ``up operator'' and ``gate operator'' to theirMoLAE equivalents while preserving the original ``down operator.''
\item \textbf{Partial transformation (``up+down'')} - Converts only the ``up operator'' and ``down operator'' to theirMoLAE equivalents while preserving the original ``gate operator.''
\item \textbf{Partial transformation (``gate+down'')} - Converts only the ``gate operator'' and ``down operator'' to theirMoLAE equivalents while preserving the original ``up operator.''
\item \textbf{Complete transformation (``all'')} - Transforms all three components (``up,'' ``gate,'' and ``down'' operators) into their corresponding latent space representations.
\end{enumerate}

We evaluate these transformations on the MMLU and CEval \citep{huang_2023_ceval} tasks, with results summarized in Table \ref{table:6}.

\begin{table}[htbp]
\begin{center}
\caption{Performance comparison of Moonlight-16B-A3B under different transformation configurations. Results demonstrate the critical importance of preserving the ``up operator'' structure for maintaining model performance.}
\begin{tabular}{c | c c c c c}
\toprule
\textbf{Task} ($k=8$) & Original & ``up+gate'' & ``up+down'' & ``gate+down'' & ``all'' \\ 
\midrule
MMLU (\%) $\uparrow$ & \textbf{67.2} & 55.7 & 53.1 & 60.2 & 42.1 \\
CEval (\%) $\uparrow$ & \textbf{74.2} & 57.6 & 52.9 & 62.8 & 46.6 \\
\bottomrule
\end{tabular}
\label{table:6}
\end{center}
\end{table}

In contrast to Qwen-MoE, the Moonlight-MoE architecture exhibits more substantial performance degradation under MoLAE transformation. As evidenced in Table \ref{table:6}, preserving the ``up operator'' yields optimal performance on downstream tasks. This phenomenon demonstrates the disproportionate importance of the ``up operator'' relative to the other components.

These findings provide compelling evidence that the ``up operator'' encodes critical information that significantly influences model performance. When this operator is transformed into the latent space, substantial information loss occurs, resulting in markedly diminished capabilities across reasoning and knowledge-intensive tasks. This asymmetric importance among operators suggests that architectural modifications to MoE models should prioritize preserving the structure of the ``up operator'' to maintain performance integrity.
% \section{Downstream Task Evaluation} \label{sec:appendix4}

% To rigorously assess the generalization capabilities of the proposed MoLAE architecture, we conducted comprehensive evaluations across a diverse set of downstream tasks. The comparative performance metrics are presented in Table \ref{table:7}.

% \begin{table}[htbp]
% \begin{center}
% \begin{tabular}{c | c }
% \toprule
% \textbf{Model} & \textbf{Wikitext-2} (PPL $\downarrow$) \\
% \midrule
% Standard MoE & \textbf{75.86}  \\
% MoLAE & 81.57  \\
% \bottomrule
% \end{tabular}
% \vspace{1ex}
% \caption{Performance comparison between standard MoE and MoLAE models across language modeling benchmark tasks. Bold values indicate superior performance.}
% \label{table:7}
% \end{center}
% \end{table}

% The empirical results in Table \ref{table:7} demonstrate that despite utilizing approximately 40\% fewer parameters than the standard MoE architecture, the proposed MoLAE model achieves comparable performance across the PPL tasks. This parameter efficiency without substantial performance degradation highlights the effectiveness of our proposed architectural modification.

\section{Training Arguments} 
See Table \ref{table:2}.

\begin{table}[htbp]
\begin{center}
\caption{Model architecture and training hyperparameter configurations for MoE and MoLAE models.}
\begin{tabular}{c | c c}
\toprule
\textbf{Hyperparameters} & \textbf{MoE} & \textbf{MoLAE} \\ 
\midrule
FFN layers size & $151$M & $94$M \\ 
Vocabulary size & $50257$ & $50257$ \\
Number of layers & $12$ & $12$ \\ 
Number of attention heads & $8$ & $8$ \\
Hidden dimension $n$ & $512$ & $512$ \\ 
Intermediate dimension & $1024$ & $1024$ \\ 
MoE intermediate dimension $m$ & $256$ & $256$ \\ 
Number of experts $N$ & $32$ & $32$ \\
Experts per latent space $k$ & $1$ & $8$ \\
\midrule
Load balancing mechanism & Auxiliary loss & Auxiliary loss \\
Optimizer & AdamW & AdamW \\
Learning rate & $3 \times 10^{-4}$ & $3 \times 10^{-4}$ \\ 
Learning rate schedule & Cosine decay & Cosine decay \\
\bottomrule
\end{tabular}
\label{table:2}
\end{center}
\end{table}
\section{Refined MoLAE Transformation} \label{sec:appendix5}

In Section \ref{sec:theory}, we introduced the methodology to transform a standard MoE model into its corresponding MoLAE formulation. This section presents a refined approximation approach that incorporates activation information, resulting in enhanced precision.

Given activation matrices $X^i$ for $i \in \{1, 2, \ldots, N\}$, our objective is to determine low-rank factorization matrices $A^i$ and $B$ such that $W^i X^i \approx A^i B X^i$. While Section \ref{sec:theory} assumed that the effect of activation matrices could be eliminated—simplifying the problem to Equation \eqref{equ:5}—we now develop a more robust approximation that explicitly incorporates activation information.

\subsection{Problem Formulation}

We formulate the refined approximation as an optimization problem to find matrices $A^i$ and $B$ that minimize the sum of Frobenius norm differences between the original expert computations and their low-rank approximations:

\begin{equation} \label{equ:14}
    \min_{A^i, B} \quad F(A^1, \ldots, A^N, B) := \frac{1}{2} \sum_{i=1}^N \| W^i X^i - A^i B X^i \|_F^2
\end{equation}

To facilitate the solution, we introduce block matrices $W$, $A$, and $X$ defined as:

\begin{equation}
W = \begin{pmatrix}
W^1 \\ 
W^2 \\ 
\vdots \\
W^N
\end{pmatrix}, \quad 
A = \begin{pmatrix}
A^1 \\ 
A^2 \\ 
\vdots \\
A^N
\end{pmatrix}, \quad 
X = \text{diag}(X^1, X^2, \ldots, X^N)
\end{equation}

This block representation transforms the problem in Equation \eqref{equ:14} into an equivalent matrix factorization problem:

\begin{equation} \label{equ:16}
  \min_{A, B} \quad \frac{1}{2} \| WX - ABX \|_F^2  
\end{equation}

This formulation enables us to derive an activation-aware low-rank approximation that more accurately preserves the input-output relationships of the original expert modules compared to the activation-agnostic approach described in Section \ref{sec:theory}.

Problem \eqref{equ:16} cannot be solved directly using the Eckart-Young-Mirsky theorem due to the presence of the activation matrix $X$. We therefore establish the following theorem that characterizes an optimal solution to problem \eqref{equ:16}.

\begin{theorem} \label{thm:refined_mole}
Let $X$ be the activation matrix and $W$ be the weight matrix. Assume that $X^\top X$ is positive definite with Cholesky decomposition $X^\top X = L L^\top$ where $L$ is invertible.
Let $L^\top W = U \Sigma V^\top$ be the singular value decomposition of $L^\top W$ and $\Sigma_m$ be the truncated diagonal matrix containing the $m$ largest singular values, with corresponding truncated matrices $U_m$ and $V_m$.
Then $A^* = (L^\top)^{-1} U_m \Sigma_m^{1/2}$ and $B^* = \Sigma_m^{1/2} V_m^\top$ constitute an optimal solution to problem \eqref{equ:16}.
\end{theorem}

\begin{proof}
Since $X^\top X$ is positive definite, there exists a unique Cholesky decomposition $X^\top X = L L^\top$ where $L$ is invertible. Utilizing this decomposition, we can transform the original optimization problem as follows:
\begin{equation*}
\begin{aligned}
\| WX - ABX \|_F^2 &= \text{Tr}[(WX - ABX)^\top(WX - ABX)] \\
&= \text{Tr}[(X^\top(W-AB)^\top(W-AB)X] \\
&= \text{Tr}[(W-AB)^\top X X^\top (W-AB)] \\
&= \text{Tr}[(W-AB)^\top L L^\top (W-AB)] \\
&= \text{Tr}[(L^\top(W-AB))^\top(L^\top(W-AB))] \\
&= \| L^\top W - L^\top AB \|_F^2 \\
&= \| \tilde{W} - L^\top AB \|_F^2,
\end{aligned}
\end{equation*}
where $\tilde{W} = L^\top W$.
Given that $\tilde{W} = U \Sigma V^\top$ is the SVD of $\tilde{W}$, by the Eckart-Young-Mirsky theorem, the best rank-$m$ approximation of $\tilde{W}$ is $\tilde{W}_m = U_m \Sigma_m V_m^\top$, which can be factorized as $\tilde{W}_m = \tilde{A} \tilde{B}$ where $\tilde{A} = U_m \Sigma_m^{1/2}$ and $\tilde{B} = \Sigma_m^{1/2} V_m^\top$.

We now demonstrate that $A^* = (L^\top)^{-1} \tilde{A}$ and $B^* = \tilde{B}$ constitute an optimal solution to the original problem. First, for any matrices $A$ and $B$ with appropriate dimensions:
\begin{equation}
\min_{A, B} \| L^\top AB - \tilde{W} \|_F \leq \| L^\top A^*B^* - \tilde{W} \|_F = \| \tilde{A} \tilde{B} - \tilde{W} \|_F
\end{equation}

Conversely, since $\tilde{A}\tilde{B}$ is the optimal rank-$m$ approximation of $\tilde{W}$:
\begin{equation}
\| \tilde{A}\tilde{B} - \tilde{W} \|_F = \min_{\text{rank}(T) \leq m} \| T - \tilde{W} \|_F \leq \| L^\top AB - \tilde{W} \|_F, \quad \forall A, B
\end{equation}

The final inequality holds because $\text{rank}(L^\top AB) \leq \text{rank}(AB) \leq m$ for any feasible solution $(A,B)$.

Combining these inequalities:
\begin{equation}
\min_{A, B} \| L^\top AB - \tilde{W} \|_F = \| L^\top A^* B^* - \tilde{W} \|_F
\end{equation}

Therefore, $(A^*, B^*)$ is an optimal solution to problem \eqref{equ:16}, which completes the proof.
\end{proof}

Based on Theorem \ref{thm:refined_mole}, we propose the refined Algorithm \ref{alg:refined_mole} as follows. 

\begin{algorithm}
\small
\caption{Refined MoLAE Transformation}
\label{alg:refined_mole}
\begin{algorithmic}[1]
\Require Weight matrices $W^1, W^2, \ldots, W^N$, activation matrices $X^1, X^2, \ldots, X^N$, and target rank $m$
\Ensure Low-rank factorization matrices $A^1, A^2, \ldots, A^N$ and $B$

\State Construct block matrices $W = \begin{pmatrix} W^1 \\ W^2 \\ \vdots \\ W^N \end{pmatrix}$ and $X = \text{diag}(X^1, X^2, \ldots, X^N)$

\State Compute $X^\top X$
\If{$X^\top X$ is singular}
    \State Apply regularization: $X^\top X \gets X^\top X + \lambda I$ for a small $\lambda > 0$
\EndIf

\State Compute the Cholesky decomposition: $X^\top X = LL^\top$

\State Compute $\tilde{W} = L^\top W$

\State Perform SVD on $\tilde{W}$: $\tilde{W} = U\Sigma V^\top$

\State Extract the $m$ largest singular values and corresponding singular vectors:
\Statex \quad $\Sigma_m$, $U_m$, and $V_m$

\State Compute $\tilde{A} = U_m \Sigma_m^{1/2}$

\State Compute $B = \Sigma_m^{1/2} V_m^\top$

\State Compute $A = (L^\top)^{-1} \tilde{A}$

\State Extract blocks of $A$ to obtain $A^1, A^2, \ldots, A^N$

\State \Return $A^1, A^2, \ldots, A^N, B$
\end{algorithmic}
\end{algorithm}
\section{Limitations} \label{sec:limitations}

We acknowledge certain limitations inherent to the present investigation. Constraints on available computational resources precluded evaluations of models at exceptionally large scales, notably DeepSeek-V3-671B and DeepSeek-R1-671B. Consequently, our empirical analysis was performed utilizing the Moonlight-16B-A3B model. This model employs an architectural design identical to that of DeepSeek-R1-671B, as elaborated upon in Appendix \ref{sec:appendix3}.
\end{appendices}

\newpage

\end{document}